\documentclass[letterpaper, 10 pt, conference]{ieeeconf}  

\IEEEoverridecommandlockouts                              

\usepackage{amsmath,amsfonts}
\usepackage{algorithm}
\usepackage{array}
\usepackage{textcomp}
\usepackage{stfloats}
\usepackage{url}
\usepackage{verbatim}
\usepackage{graphicx}
\usepackage{cite}

\usepackage{bm}
\usepackage{mathtools}
\usepackage{xcolor}
\usepackage{upgreek}
\usepackage[capitalize]{cleveref}
\usepackage{tikz}
\usetikzlibrary{arrows.meta, angles, quotes, babel, calc, shapes.misc}
\usepackage{pgfplots}
\usepackage{calrsfs}
\usepackage{amssymb}
\usepackage{subcaption}
\usepackage{csvsimple}
\usepackage{pgfplotstable}
\usepackage{booktabs}
\usepackage{makecell}
\usepackage{cellspace}
\usepackage{multicol}
\usepackage{comment}
\usepackage{threeparttable}
\usepackage{algpseudocode}
\usepackage{nccmath}
\usepackage[export]{adjustbox}
\usepackage{transparent}

\DeclareMathAlphabet{\pazocal}{OMS}{zplm}{m}{n}

\newtheorem{assumption}{Assumption}

\newtheorem{theorem}{Theorem}
\newtheorem{lemma}{Lemma}

\title{\LARGE \bf
Data-driven Force Observer for Human-Robot Interaction with Series Elastic Actuators using Gaussian Processes
}

\author{Samuel Tesfazgi$^{1,*}$, Markus Keßler$^{1,*}$, Emilio Trigili$^{2}$, Armin Lederer$^{3}$ and Sandra Hirche$^{1}$
\thanks{This work was supported by the Horizon 2020 research and innovation program of the European Union under grant agreement no. 871767 of the project ReHyb.}
\thanks{$^{1}$ S. Tesfazgi, M. Keßler and S. Hirche are with the Chair of Information-oriented Control (ITR), TUM School of Computation, Information and Technology, Technical University of Munich, 80333 Munich, Germany {\tt\small [samuel.tesfazgi, markus.kessler, hirche]@tum.de} \looseness=-1}%
\thanks{$^{2}$ E. Trigili is with the BioRobotics Institute, Scuola Superiore Sant’Anna, 56025 Pontedera, Italy {\tt\small emilio.trigili@santannapisa.it}}%
\thanks{$^{3}$ A. Lederer is with the Learning \& Adaptive Systems Group, Department of Computer Science, ETH Zurich, 8092 Zurich, Switzerland {\tt\small armin.lederer@inf.ethz.ch}}%
\thanks{$^{*}$ These authors contributed equally.}
}

\begin{document}

\maketitle
\thispagestyle{empty}
\pagestyle{empty}

\begin{abstract}

Ensuring safety and adapting to the user's behavior are of paramount importance in physical human-robot interaction. Thus, incorporating elastic actuators in the robot's mechanical design has become popular, since it offers intrinsic compliance and additionally provide a coarse estimate for the interaction force by measuring the deformation of the elastic components. While observer-based methods have been shown to improve these estimates, they rely on accurate models of the system, which are challenging to obtain in complex operating environments. In this work, we overcome this issue by learning the unknown dynamics components using Gaussian process (GP) regression. By employing the learned model in a Bayesian filtering framework, we improve the estimation accuracy and additionally obtain an observer that explicitly considers local model uncertainty in the confidence measure of the state estimate. Furthermore, we derive guaranteed estimation error bounds, thus, facilitating the use in safety-critical applications. We demonstrate the effectiveness of the proposed approach experimentally in a human-exoskeleton interaction scenario.

\end{abstract}

\section{INTRODUCTION}

Robots are increasingly deployed in application scenarios, 
where physical interaction with a human operator is required. 
In order to achieve safe interaction forces, robots with series elastic actuators (SEAs) are a popular choice and have been employed in safety-critical domains such as physical human-robot interaction \cite{yu2015seadob}. Since SEAs decouple motor and load-side through an elastic spring, compliance is inherently guaranteed. 
Moreover, the elongation of the spring is often used as a measure of interaction force \cite{sea95}, which is particularly useful in rehabilitation robots for instance to quantify the patient's participation, e.g., in assist-as-needed schemes \cite{emilio}.
However, while the spring elongation is readily available in SEAs, 
it may not reflect an accurate measure of the external force if dynamic effects are acting on the load-side \cite{park2017seakf}. Furthermore, it is challenging to isolate specific contributions to the spring interaction torque, if multiple forces are acting on the SEA load-side, e.g., during human-robot co-manipulation \cite{manipulation}. Thus, the use of external force observers is proposed to mitigate these issues. 

A widely used approach is the Kalman filter (KF), which allows to consider uncertainty in the estimation by adopting a Bayesian framework \cite{kalman1961new}. 
In \cite{liu2019momentumkf} for instance, a KF is used to estimate the generalized momentum of an upper-limb exoskeleton to infer human interaction forces. Another work 
explicitly exploits additional information encoded in the spring torque of a SEA, by augmenting the state vector by the external load-side torque and utilizing the KF as a full-state observer for estimation \cite{park2017seakf}. However, while the KF allows to consider modelling errors through process noise, inaccurate dynamics models still lead to a deteriorating quality of the estimation \cite{liu2021sensorless}. This is especially prevalent in physical human-robot interaction, as the dynamics are strongly influenced by the human, thus, requiring adaption to the individual user. \looseness=-1

A common approach to achieve this adaptation relies on learning a model of the unknown dynamics, which can in turn be used in a learning-augmented, model-based observer. 
Gaussian process (GP) regression 
has become popular for learning such models in recent years since it exhibits strong theoretical guarantees such as an inherent variance-bias trade-off, a high data-efficiency and a strong expressivity \cite{Rasmussen2006}. Moreover, it explicitly \textcolor{black}{provides} an uncertainty representation and admits the derivation of prediction error bounds \cite{GPbounds}, which is particularly useful in safety-critical applications. 
While GPs have been employed in the context of momentum observers to estimate external forces \cite{santina2020softrobotest, Evangelisti}, to the best of the authors' knowledge there is no work using GPs in a KF framework to derive a force observer for robots with series elastic actuators.
The general approach of combining GP regression and Bayesian filtering (GP-BayesFilter) is introduced in \cite{ko2009gp}, where the motion model of a robotic blimp is learned using GP regression. 
However, existing GP-BayesFilters lack theoretical guarantees on the achieved estimation error, which are vital in human-robot interaction. 

In this work, we propose a novel approach for designing a learning-augmented force observer by combining GP regression with Bayesian filtering to learn residual load-side dynamics of a SEA robot and augmenting the KF prediction model accordingly. Differently to \cite{ko2009gp}, we do not learn a state transition residual directly, but instead propose a data-generation architecture that allows to infer residual dynamics instead, which additionally lends itself for online model inference. Finally, we address the lack of theoretical estimation guarantees of GP-BayesFilters by proposing a novel approach of combining GP prediction error bounds with ellipsoid set bounds of KFs \cite{noack2010bounding}. We demonstrate the effectiveness of the proposed approach in human-exoskeleton experiments of a SEA robot with one degree of freedom.

\section{Problem Statement}

We consider the physical interaction between a human operator and an elastic joint robot, e.g., an exoskeleton with series elastic actuators (SEAs). The dynamics of such a rigid link robot with $n\in \mathbb{N}$ elastic joints is described by a set of Euler-Lagrange equations\footnote{Notation:
Lower/upper case bold symbols denote vectors/matrices, 
$\mathbb{R}_+$/$\mathbb{N}_+$ all real/integer positive 
numbers, $\bm{I}_{n\times n}$ the $n\times n$ identity matrix, $\bm{0}_{n\times n}$ the $n\times n$ matrix with zero entries and $\oplus$ the Minkowski sum. We denote an ellipsoidal set with $\pazocal{E}(\bm{c},\bm{X})$ with center $\bm{c}$ and shape matrix $\bm{X}$ define the set as $\pazocal{E}(\bm{c},\bm{X})= \{ \boldsymbol{x}\in \mathbb{R}^n | (\boldsymbol{x}- {\boldsymbol{c}})^T \boldsymbol{X}^{-1} (\boldsymbol{x}- {\boldsymbol{c}}) \leq 1 \}$. } \cite{Spong1987}
\begin{subequations}
\label{eq:sea}
    \begin{align} 
    \bm{M}(\bm{q})\ddot{\bm{q}} + \bm{C}(\bm{q},\dot{\bm{q}})+ \bm{N}(\bm{q},\dot{\bm{q}}) + \bm{\tau}_s(\bm{\theta}_s,\dot{\bm{\theta}}_s)  &= \bm{\tau}_{\text{ext}}
    \label{eq:loadside} \\
    \bm{J}(\bm{\theta}_m)\ddot{\bm{\theta}}_m + \bm{D}_m \dot{\bm{\theta}}_m - \bm{\tau}_s(\bm{\theta}_s,\dot{\bm{\theta}}_s) &= \bm{\tau}_m, 
    \label{eq:motorside}
    \end{align}
\end{subequations}
where \eqref{eq:loadside} describes the load-side dynamics and \eqref{eq:motorside} the motor-side dynamics. Here, $\bm{q} \in \mathbb{R}^n$ represents the joint angles, $\bm{\theta}_m \in \mathbb{R}^n$ represents the motor angles, ${\bm{M}(\bm{q})\in \mathbb{R}^{n\times n}}$ is the inertia matrix of the rigid link, $\bm{J}(\bm{\theta}_m)\in\mathbb{R}^{n\times n}$ is the inertia matrix of the motors, $\bm{C}(\bm{q},\dot{\bm{q}})\in \mathbb{R}^{n\times n}$ denotes Coriolis, centrifugal and gravitational terms on the load side, ${\bm{D}_m \in \mathbb{R}^{n\times n}}$ denotes the motor damping matrix and $\bm{N}(\bm{q},\dot{\bm{q}})\in \mathbb{R}^{n}$ represents other, lumped nonlinear effects on the load-side such as friction. The column vector $\bm{\tau}_m$ describes the torque inputs provided by the motors, while $\bm{\tau}_\text{ext}$ corresponds to any external torques acting on the load side, e.g., environmental disturbances or human torques. Both load and motor side are coupled through the elastic transmission torque \looseness=-1
\begin{equation}
    \bm{\tau}_s(\bm{\theta}_s,\dot{\bm{\theta}}_s) = \bm{K}_s(\bm{\theta}_s,\dot{\bm{\theta}}_s) \bm{\theta}_s + \bm{D}_s(\bm{\theta}_s,\dot{\bm{\theta}}_s) \dot{\bm{\theta}}_s,
    \label{eq:spring_torque}
\end{equation}
which is dependent on the spring deformation $\bm{\theta}_s=\bm{q}-\bm{\theta}_m$ with nonlinear stiffness $\bm{K}_s \in \mathbb{R}^{n\times n}$ and damping matrix ${\bm{D}_s \in \mathbb{R}^{n\times n}}$. \looseness=-1 
The goal is to estimate the \emph{actively} generated torques of the human operator acting on the load-side of the elastic joint robot. To this end, we decompose the external torques $\bm{\tau}_\text{ext}$ into passive torques ${\bm{\tau}_\text{h,pas}(\bm{q},\dot{\bm{q}},\ddot{\bm{q}}) \colon \mathbb{R}^n\times\mathbb{R}^n\times\mathbb{R}^n \to \mathbb{R}^n}$ due to inertial, gravitational and viscoelastic torques of the human limb and active torques ${\bm{\tau}_\text{h,act}\in\mathbb{R}^n}$ generated by volitional muscle activity, i.e., ${\bm{\tau}_\text{ext}= \bm{\tau}_\text{h,pas}+\bm{\tau}_\text{h,act}}$. Since the precise identification of the parameters governing the human passive dynamics is challenging, we merely assume that an approximate model of ${\bm{\tau}_{\text{h,pas}}}$ is available, e.g., using anthropometric tables \cite{just2020human}. Also, we assume that the nonlinear effects $\bm{N}(\bm{q},\dot{\bm{q}})$ on the load-side are unknown. 
This restriction reflects the fact that it is generally difficult to model effects such as friction using parametric approaches. Thus, we can lump all uncertain dynamics components into the unknown function $\bm{f}(\bm{z})=\bm{N}(\bm{q},\dot{\bm{q}})-\bm{\tau}_\text{h,pas}(\bm{q},\dot{\bm{q}},\ddot{\bm{q}}) $, where $\bm{z}\coloneqq \big[\bm{q}^\intercal\, \dot{\bm{q}}^\intercal \, \ddot{\bm{q}}^\intercal\big]^\intercal $ is the concatenation of joint angles, velocities and acceleration, such that we can write \eqref{eq:loadside} as 
\begin{equation}
    \bm{M}(\bm{q})\ddot{\bm{q}} + \bm{C}(\bm{q},\dot{\bm{q}})+ \bm{f}(\bm{z}) + \bm{\tau}_s(\bm{\theta}_s,\dot{\bm{\theta}}_s)  = \bm{\tau}_{\text{h,act}},
    \label{eq:unknwn_load} 
\end{equation}
To obtain an accurate estimate of $\bm{\tau}_{\text{h,act}}$  
despite the residual dynamics error, we assume access to the following measurements to infer a model of $\bm{f}(\cdot)$. \looseness=-1
\begin{assumption}
    \label{ass:data}
    The motor torque $\bm{\tau}_m$, the load-side kinematics $\{\bm{q}, \dot{\bm{q}}, \ddot{\bm{q}}\}$ and the motor kinematics $\{\bm{\theta}_m, \dot{\bm{\theta}}_m, \ddot{\bm{\theta}}_m\}$ are available 
    for model inference.
\end{assumption}

The motor torque $\bm{\tau}_m$ is directly computable from the applied current and motor constant. Moreover, SEAs are typically equipped with encoders on load- and motor-side such that $\bm{q}$ and $\bm{\theta}_m$ are measurable and angular velocities and accelerations can be obtained through numerical differentiation. Note that we do not assume any force/torque measurements on the load-side, which would require expensive sensors that additionally often suffer from measurement noise. 

Based on the above, we consider the problem of estimating the active human torque $\bm{\tau}_\text{h,act}$ using a model-based observer. Since the estimation of $\bm{\tau}_\text{h,act}$ requires an accurate model of the human-robot system, we learn the unknown dynamics component, including the human passive dynamics, and augment the observer with the inferred $\bm{f}(\cdot)$. 
Finally, due to the safety-critical application scenario, guarantees for the estimate $\bm{\hat{\tau}}_\text{h,act}$ 
should be obtained. We aim to derive guarantees in the form of probabilistic estimation error bounds
\begin{equation}
    Pr\Big(|\hat{\bm{\tau}}_{\text{h,act}}-\bm{\tau}_{\text{h,act}}|\leq \rho \Big) \geq 1-\delta,
    \label{eq:bounds}
\end{equation}
with the upper bound $\rho\in\mathbb{R}^+$ evaluated for a probability level determined by $\delta\in(0,1)$. The bound \eqref{eq:bounds} represents the confidence in the estimated torque $\bm{\tau}_\text{h,act}$ 
and ensures reliable use in physical human-exoskeleton interaction. 

\textit{Remark:} While we consider the estimation of active human torques in this work, the proposed approach is not limited to this application and may also be employed as a general external force observer. In this case, the unknown dynamics are primarily introduced by unmodelled component on the load-side of the robot, e.g., due to nonlinear friction. Thus, no decomposition of the external force $\bm{\tau}_{\text{ext}}$ is required.

\section{Preliminaries}

To estimate the active human torque, we employ Bayesian filtering together with GP regression. Fundamentals of state estimation using Bayesian filtering are introduced in \Cref{subsec:bayesian}, before explaining GP regression in \Cref{subsec:gpreg}. 

\subsection{State Estimation using Bayesian Filtering} \label{subsec:bayesian}
A well-studied approach to estimate non-measurable quantities of interest is by means of Kalman filter (KF). 
While the KF was first developed for linear systems \cite{kalman1961new}, several extensions to nonlinear systems, e.g., the extended KF (EKF)
\cite{radke2006survey}, have been proposed since. Due to the nonlinear dynamics of human-exoskeleton system, we adopt the EKF framework in the following. Consider a nonlinear system
\begin{align}
\label{eq:nonlinkfsys}
        \bm{x}_{k+1} &= \bm{f}_k(\bm{x}_{k},\bm{u}_k) + \bm{w}_{k}  \\
        \bm{y}_{k} &= \bm{H}\bm{x}_{k}+ \bm{v}_{k}
\label{eq:nonlinkOsys}
\end{align}
at time steps $k \in \mathbb{N}^+$ with system states $\bm{x} \in \mathbb{R}^d$, control input $\bm{u} \in \mathbb{R}^r$ and measurement vector ${\bm{y} \in \mathbb{R}^m}$ as well as nonlinear state transitions ${\bm{f}\colon\mathbb{R}^d\times\mathbb{R}^r\to\mathbb{R}^d}$ and linear observation model $\bm{H}\in\mathbb{R}^{m\times d}$. 
The process $\bm{w}_{k} \sim \pazocal{N}(\bm{0},\bm{Q}_k)$ and measurement noise $\bm{v}_{k} \sim \pazocal{N}(\bm{0},\bm{R}_k)$ follow a Gaussian distribution with variance $\bm{Q}_k\in \mathbb{R}^{d\times d}$ and $\bm{R}_k\in \mathbb{R}^{m\times m}$, and are assumed uncorrelated and white, i.e., $\mathbb{E}[\bm{v}_k, \bm{w}_{k}] = 0 \ \forall k$, $\mathbb{E}[\bm{v}_k, \bm{v}_{j}] = 0 \ \forall k \neq j$ and $\mathbb{E}[\bm{w}_k, \bm{w}_{j}] = 0 \ \forall k \neq j$. Note that in the considered SEA scenario, the measurement model $\bm{H}$ is typically linear, thus, \eqref{eq:nonlinkOsys} does not pose a restriction. 

The EKF provides estimates in the form of a mean state estimate $\hat{\bm{x}}_{k}$ together with an uncertainty description through the covariance $\bm{P}_k\in\mathbb{R}^{d\times d}$. In particular, the EKF follows a Bayesian approach where a \textit{prior} distribution, i.e., prior mean $\hat{\bm{x}}_{k}^-$ and covariance $\bm{P}_{k}^-$, are predicted using the dynamics model \eqref{eq:nonlinkfsys} and the previous state estimate $\hat{\bm{x}}_{k-1}$, which are then updated by conditioning the \textit{prior} distribution on measurements $\bm{y}_k$ to obtain a \textit{posterior} distribution $\hat {\bm {x} }_{k}$ and $\bm {P} _{k}$. The complete procedure is shown in \Cref{alg:ekf}. An advantage of the KF framework is that it permits an elegant uncertainty propagation from the prediction to the update step through the Kalman gain $\bm{K}_{k}$ (Alg. \ref{alg:ekf}, line 8). 
When process noise $\bm{Q}_k$ is dominant, i.e., high model uncertainty, the prior covariance $\bm{P}_k^-$ becomes large, resulting in a high gain $\bm {K} _{k}$ and giving more weight to the measurement ${\bm{y}}_{k}$. Conversely, when measurement noise $\bm{R}_k$ dominates, the gain $\bm {K} _{k}$ is low and less weight is put to the measurement ${\bm{y}}_{k}$.

Thereby, the KF provides a convenient framework to include local model uncertainty. Thus, if the deployed prediction model \eqref{eq:nonlinkfsys} is augmented by a learned component, it can be beneficial to explicitly include the uncertainty of the inferred model in the estimation. While the process noise covariance $\bm{Q}$ may be naively set to a constant values determined by a global measure of learning error as in \cite{liu2021sensorless}, this approach cannot take local model uncertainties into account, which typically arises due to inhomogenous data distributions. On the other hand, 
learning techniques that quantify local model uncertainty facilitate dynamically updating the process noise covariance $\bm{Q}_k$ based on the local confidence of the learned model at each step $k$. A prominent learning method that quantifies uncertainty is Gaussian Process regression introduced in the following.

\begin{algorithm}[t!]
\small
\caption{Extended Kalman filter (EKF)}
\begin{algorithmic}[1]
\State \textbf{Initialize} $\hat{\bm{x}}_0$, $\bm{P}_0$

\For{k = 1,\ ...,\ K}
\Statex \quad\ \textbf{Prediction step:}
\State $\bm {F}_{k-1}=\left.{\frac {\partial \bm{f}_{k-1}}{\partial {\bm {x}}}}\right\vert _{{\hat {\bm {x}}}_{k-1},{\bm {u}}_{k-1}}$
\State $\hat{\bm{x}}_{k}^- = \bm{f}_{k-1}(\hat{\bm{x}}_{k-1},\bm{u}_{k-1})$
\State $\bm{P}_{k}^- = \bm {F}_{k-1} \bm {P}_{k-1} \bm {F} _{k-1}^T + \bm {Q} _{k}$ 
\Statex \quad\ \textbf{Update step:}
\State $\bm {K} _{k}=\bm {P} _{k}^- \bm {H}^T \big( \bm {H}\bm {P} _{k}^- \bm {H}^T +\bm {R} _{k}\big)^{-1}$
\State $\hat {\bm {x} }_{k}= \left(\bm {I}_d -\bm {K} _{k}\bm {H} _{k}\right){\hat {\bm {x} }}_{k}^- + \bm {K} _{k}{\bm {y} }_{k}$ 
\State $\bm {P} _{k}=\left(\bm {I}_d -\bm {K} _{k}\bm {H} _{k}\right)\bm {P} _{k}^-$
\EndFor
\end{algorithmic}
\label{alg:ekf}
\end{algorithm}

\subsection{Gaussian Process Regression} \label{subsec:gpreg}

    Gaussian process (GP) regression is a modern machine learning technique basing on the assumption that any finite number of evaluations
	$\{f(\bm{x}^{(1)}) \ldots f(\bm{x}^{(N)})\}$, $N\in\mathbb{N}_+$ of an unknown function $f:\mathbb{R}^{\rho}\rightarrow\mathbb{R}$, $\rho\in\mathbb{N}$, 
	at inputs $\bm{x}\in\mathbb{R}^{\rho}$ follows a joint Gaussian distribution~\cite{Rasmussen2006}. This distribution is defined using a prior mean function $m:\mathbb{R}^{\rho}\rightarrow\mathbb{R}$ and a covariance function $k:\mathbb{R}^{\rho}\times\mathbb{R}^{\rho}\rightarrow\mathbb{R}$ yielding the compact notation $\pazocal{GP}(m(\bm{x}),k(\bm{x},\bm{x}'))$. The prior mean function is commonly used to incorporate prior such as an approximate model and is commonly set to $0$, which we also assume in the following. The covariance function encodes more abstract information such as differentiability of the unknown function. Due to its infinite differentiability, the squared exponential kernel 
    \begin{align}\label{eq:SE}
	k(\bm{x},\bm{x}')=\sigma_f^2\mathrm{exp}\left(-\sum\limits_{i=1}^{\rho}\frac{(x_i-x_i')^2}{2l_i^2}\right),
	\end{align} 
    is commonly employed as covariance function, whose shape depends on the signal standard deviation ${\sigma_f\in\mathbb{R}_+}$ and the length scales $l_i\in\mathbb{R}_+$.
	Together with the target noise standard deviation $\sigma_{\mathrm{on}}\in\mathbb{R}_+$, they form the hyperparameters ${\bm{\theta}=\begin{bmatrix}
	\sigma_f& l_1 & \ldots & l_{\rho} & \sigma_{\mathrm{on}}\end{bmatrix}^\intercal}$, which are commonly obtained by maximizing the log-likelihood \cite{Rasmussen2006}.

    Given the hyperparameters, the posterior distribution conditioned on training data can be calculated under the assumption of training targets perturbed by zero mean Gaussian noise with $\sigma_{\mathrm{on}}^2$. It is straightforward to show that this posterior follows a Gaussian distribution with mean and variance  \begin{align}\label{eq:GPmean}
	\mu(\bm{x})&= \bm{y}^T\left(\bm{K}+\sigma_{\mathrm{on}}^2\bm{I}_N \right)^{-1}\bm{k}(\bm{x})\\
	\sigma^2(\bm{x})&=k(\bm{x},\bm{x})-\bm{k}^T(\bm{x})\left(\bm{K}+\sigma_{\mathrm{on}}^2\bm{I}_N \right)^{-1}\bm{k}(\bm{x}),
	\label{eq:GPvariance}
	\end{align}
	where we define the elements of the kernel vector ${\bm{k}(\bm{x})\in\mathbb{R}^{N}}$ 
	as $k_i(\bm{x})=k(\bm{x}^{(i)},\bm{x})$.

\section{Human Torque Estimation using Gaussian Processes}
In order to seamlessly integrate the GP regression into a model-based observer, we propose the filter architecture outlined in \Cref{fig:full_scheme}. It consists of a component responsible for model inference using GPs, which is detailed in \Cref{subsec:online_learning}. The learning component is used to enhance an EKF in \Cref{subsec:GPAKF} and guaranteed estimation error bounds for the proposed GP-enhanced observer are derived in \Cref{subsec:Guarantees}.  

\begin{figure}
    \centering
        \includegraphics[width=0.45\textwidth]{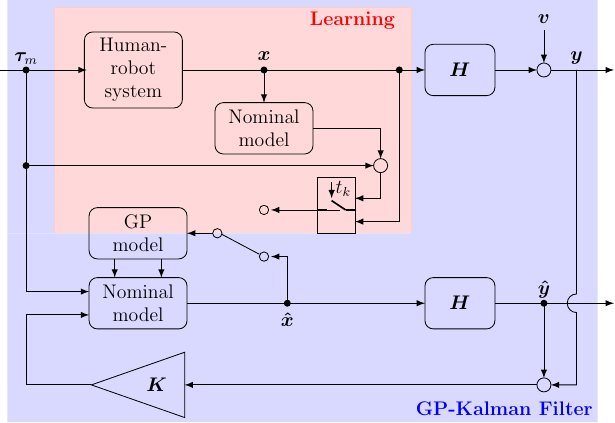}
    \caption{GP-BayesFilter architecture with learning components highlighted in red and the filtering blocks illustrated in blue.}
    \label{fig:full_scheme}
\end{figure}

\subsection{Data Generation and Model Inference} \label{subsec:online_learning}
In order to account for the unknown dynamics component $\bm{f}(\cdot)$ in \eqref{eq:unknwn_load},
it is necessary to augment the KF by a model inferred from data. 
While the use of GP models in a Bayes-filter framework by directly learning a state transition residual $\Delta\bm{x}$ is discussed in \cite{ko2009gp}, applying this approach to the SEA dynamics leads to a coupling of load-side inertia $\bm{M}(\cdot)$ and external torques $\bm{\tau}_{\text{ext}}$. Due to the considered lack of force/torque sensors as stated in \Cref{ass:data}, this coupling cannot be resolved and the inertia error would persist. 

To overcome this problem we 
learn the residual torque $\bm{f}(\cdot)$ instead. To this end, we rearrange the motor-side dynamics \eqref{eq:motorside} and substitute the spring torque $\bm{\tau}_s(\cdot,\cdot)$ into the load-side \eqref{eq:unknwn_load} yielding the inverse dynamics expression
\begin{equation}
    	\bm{f}(\bm{z}) = \bm{\tau}_\text{h,act} + \bm{\tau}_m - \big(\bm{\tau}_l(\bm{z}) + \bm{J}(\bm{\theta}_m)\ddot{\bm{\theta}}_m + \bm{D}_m \dot{\bm{\theta}}_m\big),
	\label{eq:loadsidewmotor}
\end{equation}
where $\bm{\tau}_l(\bm{z}) \coloneqq \bm{M}(\bm{q})\ddot{\bm{q}} + \bm{C}(\bm{q},\dot{\bm{q}})$. 
We make following assumption for $\bm{\tau}_\text{h,act}$  
during the data generation.
\begin{assumption}\label{ass:passive_human}
    The human operator is assumed passive during training, i.e., the actively generated torque $\bm{\tau}_\text{h,act}=\bm{0}$. 
\end{assumption}

Similar assumptions are common in learning-augmented observer design \cite{santina2020softrobotest,liu2021sensorless,Evangelisti} and are necessary to decompose discrepancies due to modelling errors and external disturbances, which is an inherently ill-posed problem \cite{illposed}. 
Since the motor torque $\bm{\tau}_m$, load-side kinematics $\{\bm{q}, \dot{\bm{q}}, \ddot{\bm{q}}\}$ and motor kinematics $\{\bm{\theta}_m, \dot{\bm{\theta}}_m, \ddot{\bm{\theta}}_m\}$ are measurable due to \Cref{ass:data} and 
a model of the 
SEA dynamics is available
, the right hand side of \eqref{eq:loadsidewmotor} can directly be computed given \Cref{ass:passive_human}. Therefore, it remains to define a sampling rate $1/t$, $t\in \mathbb{R}_+$ at which measurements of $\bm{z}$ are taken and values $\bm{f}(\cdot)$ are computed, such that a training data set 
\begin{align}
	\{(\bm{x}^{(k)}=\bm{z}(kt), \enspace \bm{y}^{(k)}=\bm{f}(kt)\}_{k=0}^K
\end{align}
is aggregated based on the measurements. Using the data set, we can update an independent GP for each target dimension of $\bm{y}^{(k)}\!$, i.e., for each $i\!=\!1,\ldots,n$ a GP is updated using a training pair $(\bm{x}^{(k)},y_i^{(k)})$. In order to employ the GP models in the observer, their predictions and variances are concatenated into a vector ${\bm{\mu}}(\bm{x})\!=\![\mu_1(\bm{x})\ \cdots\  \mu_n(\bm{x})]^T$ and $\bm{\sigma}^2(\bm{x})=[\sigma^2_1(\bm{x})\ \cdots\ \sigma^2_n(\bm{x})]^T$, respectively. 

\subsection{Learning-based Augmented-State Torque Observer} \label{subsec:GPAKF}
To 
enhance the prediction of the observer using the learned model,  we deploy an augmented-state Kalman filter (AKF), where we define the augmented state vector $\bm{x}\in\mathbb{R}^{5n}$ as 
\begin{equation}
    \bm{x} = \begin{bmatrix}\bm{\theta}_m^\intercal & \bm{\theta}_s^\intercal &\dot{\bm{\theta}}_m^\intercal & \dot{\bm{\theta}}_s^\intercal & \bm{\tau}_{\text{h,act}}^\intercal \end{bmatrix}^\intercal.
\end{equation}
Rearranging the motor-side dynamics \eqref{eq:motorside} yields
\begin{equation}
    \ddot{\bm{\theta}}_m=\bm{J}^{-1}(\bm{\tau}_m+\bm{\tau}_s(\bm{\theta}_s,\dot{\bm{\theta}}_s)-\bm{D}_m \dot{\bm{\theta}}_m).
    \label{eq:thetadd}
\end{equation}
Substituting \eqref{eq:thetadd} in \eqref{eq:unknwn_load} and applying ${\bm{\theta}_s\!=\!\bm{q}\!-\!\bm{\theta}_m}$, we get
\begin{equation}
\begin{split}
        \ddot{\bm{\theta}}_s=\bm{M}^{-1}\left(\bm{\tau}_{\text{h,act}} - \bm{f}(\bm{z}) - \bm{C}(\bm{x}) -\bm{\tau}_s(\bm{\theta}_s,\dot{\bm{\theta}}_s)\right)\\
        -\bm{J}^{-1}\left(\bm{\tau}_m+\bm{\tau}_s(\bm{\theta}_s,\dot{\bm{\theta}}_s)-\bm{D}_m \dot{\bm{\theta}}_m\right).
\end{split}
\label{eq:xsdd}
\end{equation}
Thus, based on \eqref{eq:thetadd} and \eqref{eq:xsdd}, the system state is driven by the nonlinear dynamics 
\begin{equation}
    \dot{\bm{x}}=\bm{f}_{\mathrm{nom}}(\bm{x},\bm{u})-\bm{I}_{f}\bm{M}^{-1}\bm{f}(\bm{z}), \label{eq:unknown}
\end{equation}
where, omitting the dependencies on entries of $\bm{x}$ for improved readability, we define the nominal dynamics
\begin{align}\label{eq:nonlindyn}
    &\bm{f}_{\mathrm{nom}}(\bm{x},\bm{u})=\\
    &\begin{bmatrix} 
\dot{\bm{\theta}}_m \\
\dot{\bm{\theta}}_s \\
\bm{J}^{-1}(\bm{\tau}_m+\bm{\tau}_s-\bm{D}_m \dot{\bm{\theta}}_m) \\
\bm{M}^{-1}\!\!\left(\bm{\tau}_{\text{h,act}}  \!-\! {\bm{C}} \!-\! \bm{\tau}_s\right)\!-\!\bm{J}^{-1}\!\!\left(\bm{\tau}_m\!+\!\bm{\tau}_s\!-\!\bm{D}_m \dot{\bm{\theta}}_m\right) \\
\bm{0}_{n\times 1}
\end{bmatrix},\nonumber
\end{align}
$\bm{u}=\tau_m$ and ${\bm{I}_{\mu} = [\bm{0}_{n}\ \bm{0}_{n}\ \bm{0}_{n}\ \bm{I}_{n}\  \bm{0}_{n}]^\intercal}$.
Note that \eqref{eq:nonlindyn} imposes zero-order dynamics for the active torque, i.e., ${\dot{\bm{\tau}}_{\text{h,act}}=\bm{0}}$, 
which is sufficient to model piece-wise constant or slowly varying behavior \cite{wahrburg2015momentumkf}. However, if more information is available, it is straightforward to include it by introducing a torque state $\bm{\omega}$ with arbitrary dynamics $\dot{\bm{\omega}}=s(\bm{\omega})$. \looseness=-1

To flexible adapt to the unknown dynamics, 
we exploit the learned model and substitute $\bm{f}(\cdot)$ in \eqref{eq:unknown} with the concatenation of GP predictions ${\bm{\mu}}(\cdot)$. Thus, we write the nonlinear prediction model of the GP-enhanced AKF (GP-AKF) to \looseness=-1
\begin{equation}
    \bm{f}(\bm{x},\bm{u}) = \bm{f}_{\text{nom}}(\bm{x},\bm{u}) - \bm{I}_{\mu}\bm{M}^{-1}{\bm{\mu}}(\bm{z}).
    \label{eq:fwdynwtorqueresidual1}
\end{equation}
Since the prediction model \eqref{eq:fwdynwtorqueresidual1} is nonlinear, the Jacobian of $\bm{f}(\cdot,\cdot)$ is needed to apply the EKF algorithm. It is straightforward to compute the linearization for the nominal dynamics model $\bm{f}_{\text{nom}}(\cdot,\cdot)$ as demonstrated in \cite{park2017seakf}. For the Jacobian of the GP predictions 
we obtain
\begin{equation}
    \left.{\frac {\partial {\mu}}{\partial {\bm {x}}}}\right\vert _{{\hat {\bm {x}}}} =  \left(\frac{\partial \bm{k}(\bm{x})}{\partial \bm{x}}\right)^\intercal \!\! \left(\bm{K}+\sigma_{\mathrm{on}}^2\bm{I}_N \right)^{-1} \bm{y},
    \label{eq:moejac}
\end{equation}
which is easily computable for common kernel choices, such as the squared exponential kernel \eqref{eq:SE}, since
\begin{equation}
    \frac{\partial k(\bm{x}^{(i)},\bm{x})}{\partial x^{(i)}} = \frac{x-x^{(i)}}{l^2} k(\bm{x}^{(i)},\bm{x}).
    \label{eq:kernelderivative}
\end{equation}
Then, introducing the notation ${\bm{F}_\text{nom} \coloneqq {\frac {\partial \bm{f}_{\text{nom}}}{\partial {\bm {x}}}}\big\vert_{{\hat {\bm {x}}}}}$ and ${\bm{F}_{\text{GP}}\coloneqq {\frac {\partial \tilde{\bm{\mu}}}{\partial {\bm {x}}}}\big\vert_{{\hat {\bm {x}}}}}$, we get the linearized GP-AKF dynamics
\begin{equation}
\bm{F}(\bm{x},\bm{u})=\!\bm{F}_\text{nom}(\bm{x},\bm{u})-\bm{I}_{\mu}\bm{M}^{-1}\!\bm{F}_{\text{GP}}(\bm{x},\bm{u}),
\label{eq:linearization}
\end{equation}
In order to additionally leverage the inherent uncertainty quantification provided by the GP, we incorporate the GP posterior variance ${\bm{\sigma}}^2(\cdot)$ in the EKF process noise covariance
\begin{equation}
    \bm{Q}(\bm{x}) = \bm{\Sigma}_\text{nom}+{\bm{\Sigma}}(\bm{x}),
    \label{eq:gpakf_covariance}
\end{equation}
where $\bm{\Sigma}_\text{nom}$ denotes the process noise covariance of the nominal model and ${\bm{\Sigma}}(\cdot)$ the covariance induced by the GP-model \looseness=-1
\begin{equation}
    {\bm{\Sigma}}(\bm{x})=\bm{I}_{\mu}\bm{M}^{-1}\text{diag}({\bm{\sigma}}^2(\bm{x}))\bm{M}^{-\intercal}\bm{I}^\intercal_{\mu}.
    \label{eq:GPvarfwd}
\end{equation}
Finally, we apply a time-discretization procedure, e.g., as described in \cite{wahrburg2015momentumkf}, to obtain the time-discrete, nonlinear GP-enhanced AKF. The resulting method is shown in \Cref{alg:torqueresidualAKF}, where, with a slight abuse of notation, we denote the time-discretized system using the same symbols and time-step $k$. In comparison to the EKF in \Cref{alg:ekf}, our proposed GP-AKF uses the learning-augmented model \eqref{eq:fwdynwtorqueresidual1} to predict the prior estimate $\hat{\bm{x}}_{k}^-$ (Alg. \ref{alg:torqueresidualAKF}, line 7) and exploits the uncertainty quantification of the GP in the covariance estimate (Alg. \ref{alg:torqueresidualAKF}, line 8-10). Note that using the GP-augmented covariance $\bm{Q}_k$ (Alg. \ref{alg:torqueresidualAKF}, line 9) to update the KF covariance $\bm{P}_{k}^-$ allows us to obtain an observer that explicitly considers local model uncertainty in the confidence measure of the state estimate. However, for the prediction of the prior covariance $\bm{P}_{k}^-$ (Alg. \ref{alg:torqueresidualAKF}, line 10) the linearized model $\bm{F}_k$ is used to retain a Gaussian distribution. Thereby, an error is introduced and the estimated covariance does not represent the true covariance in general. Nevertheless, it is possible to derive guaranteed estimation error bounds for the proposed GP-AKF method, which we demonstrate in the following section.

\begin{algorithm}[t!]
\small
\caption{GP-enhanced AKF (GP-AKF)}
\begin{algorithmic}[1]
\State \textbf{Initialize} $\hat{\bm{x}}_0$, $\bm{P}_0$
\State $\tilde{\bm{\mu}}_k \leftarrow \bm{0}_{n\times 1}, \tilde{\bm{\sigma}}^2_k\leftarrow \bm{0}_{n\times 1}$

\For{k = 1,\ ...,\ K}
\Statex \quad\ \textbf{Prediction step:}
\For{$i = 1,\ ...,\ n$}
    \State $\tilde{\mu}_{k,i}, {\sigma}^2_{k,i} = \texttt{GP}_i.\texttt{PREDICT}(\hat{\bm{z}}_{k-1})$
\EndFor
\State $\hat{\bm{x}}_{k}^- = \bm{f}^k_{\text{nom}} - \bm{I}_{\mu}\bm{M}^{-1}\tilde{\bm{\mu}}_k$ 
\State $\bm{F}_k =\bm{F}_\text{nom}-\bm{I}_{\mu}\bm{M}^{-1}\!\bm{F}_{\text{GP}}$ 
\State $\bm{Q}_k$ = $\bm{\Sigma}_{\text{nom},k}+\bm{I}_{\mu}\bm{M}^{-1}\text{diag}({\bm{\sigma}}^2_k)\bm{M}^{-\intercal}\bm{I}^\intercal_{\mu}$ 
\State $\bm{P}_{k}^- = \bm {F}_{k} \bm {P}_{k-1} \bm {F} _{k}^T + \bm {Q} _{k}$ 
\Statex \quad\ \textbf{Update step:}
\State identical to \cref{alg:ekf}
\EndFor
\end{algorithmic}
\label{alg:torqueresidualAKF}
\end{algorithm}

\subsection{Guaranteed Error Bounds on the Estimated Torque} \label{subsec:Guarantees}

In order to derive estimation error bounds for the learning-based observer, we integrate the confidence sets of the GP prediction into the KF framework. To this end, we adopt the elliptical set-membership approach \cite{noack2010bounding}, where the system state $\bm{x}$ is expressed as a set of Gaussian distributions
\begin{equation}
    \{\bm{x}\sim\pazocal{N}(\Bar{\bm{x}}+\bm{d},\bm{\Sigma})\  |\ \bm{d}\in\pazocal{E}(\bm{0},\bm{X})\subset \mathbb{R}^d\},
    \label{eq:gaussianset}
\end{equation}
with mean $\Bar{\bm{x}}$ and unknown but bounded perturbation $\bm{d}$ 
described by the positive semi-definite shape matrix ${\bm{X}\in \mathbb{R}^{d\times d}}$. Thus, the set \eqref{eq:gaussianset} is generated by a set of Gaussian distributions with covariance $\bm{\Sigma}$, which are centered around a set of means $\pazocal{X} = \pazocal{E}(\Bar{\bm{x}},\bm{X})$. 
Furthermore, the confidence set of a Gaussian distributed random variable $\bm{\xi} \sim \pazocal{N}(\bar{\bm{\xi}},\bm{C})$ for a probability level $\delta\in(0,1)$ is defined by an ellipsoid $\pazocal{E}(\bar{\bm{\xi}},s\bm{C})$ such that
\begin{equation}
    Pr\big(\bm{\xi}\in\pazocal{E}(\bar{\bm{\xi}},s\bm{C})\big)\geq 1-\delta
    \label{eq:gaussianellipsoid}
\end{equation}
where the scalar $s$ is determined by the chi-square distribution for the selected probability level \cite{noack2010bounding}. Thus, based on \eqref{eq:gaussianellipsoid}, an ellipsoidal confidence set $\pazocal{C}$ can be obtained for $\bm{x}$
\begin{equation}
    \pazocal{C} = \pazocal{E}(\Bar{\bm{x}},\bm{X}) \oplus \pazocal{E}(\bm{0},s\bm{\Sigma}), 
    \label{eq:confidenceset}
\end{equation}
satisfying $Pr(\bm{x}\in\pazocal{C})\!\geq\! 1\!-\!\delta$. To derive a confidence set $\pazocal{C}$ for our GP-AKF, we utilize the fact that GP regression provides guaranteed error bounds on the regressed mean function \looseness=-1 
\begin{equation}
    Pr\big(|f_i(\bm{x})-\mu_i(\bm{x}) | \leq \beta\sigma_i(\bm{x}),\ \forall \bm{x}\in\mathbb{X}\big) \geq 1-\delta,
\label{eq:gpmeanbounds}
\end{equation}
where $\mathbb{X}\subset\mathbb{R}^n$ can be an arbitrary compact set and $\beta\in\mathbb{R}_+$ is a constant \cite{Srinivas2012, GPbounds}.
To integrate the GP error bounds in the KF set-membership approach, we first reformulate the interval bound \eqref{eq:gpmeanbounds} to an ellipsoid.
\begin{lemma} \label{lem:interval_bound}
    The interval bound $|f_i(\bm{x})-\mu_i(\bm{x}) | \leq \eta$ for every $\eta\in\mathbb{R}_+$ is equivalent to the one-dimensional ellipsoid inclusion $f_i(\bm{x})\in\pazocal{E}(\mu_i(\bm{x}),\eta^{2})$.
\end{lemma}
\begin{proof}
    This can be seen directly by squaring the interval bound and dividing by the squared error bound, resulting in
    \begin{equation}
        (f_i(\bm{x})-\mu_i(\bm{x})) \,\eta^{-2} (f_i(\bm{x})-\mu_i(\bm{x})) \leq 1,
    \end{equation}
    which is exactly an ellipsoid $\pazocal{E}(\mu_i(\bm{x}),\eta^{2})$ centered around $\mu_i(\bm{x})$ with shape parameter $\eta^{2}$, concluding the proof.
\end{proof}

When considering multi-dimensional prediction with $n$ independent GPs for each target dimension, as in \Cref{subsec:online_learning}, it is straightforward to enclose the $n$ one-dimensional ellipsoids by one $n$-dimensional ellipsoid. 
To this end, we define the pair of mean vector $\Tilde{\bm{\mu}}_i = [0\,\ldots\, {{\mu}}_i \,\ldots \,0]^\intercal$ and shape matrix $\Tilde{\bm{X}}_i=\text{diag}([0\, \ldots\, \eta_i^2(\kappa,\bm{x})\, \ldots\, 0])$.
Then taking the Minkowski sum of the $n$ ellipsoids defined by the pairs $\{\Tilde{\bm{\mu}}_i, \Tilde{\bm{X}}_i\}$ yields an outer approximation
\begin{equation}
    \pazocal{E}(\Tilde{\bm{\mu}},\Tilde{\bm{X}}) \supseteq \bigoplus_{i=1}^n \pazocal{E}(\Tilde{\bm{\mu}}_i,\Tilde{\bm{X}}_i),
    \label{eq:GPboundsellipse}
\end{equation}
with shape matrix $\Tilde{\bm{X}}$ that bounds the error of the GP prediction $\Tilde{\bm{\mu}}$ with probability $1-\delta$.
Hence, using the elliptical set-membership formalism \eqref{eq:gaussianset} together with the ellipsoidal GP error bound \eqref{eq:GPboundsellipse}, the error sources of the GP-AKF can be decomposed 
for each source separately to derive a novel error bound for the overall approach. The derivations of the estimation error bounds is presented in the following result.

\begin{theorem} \label{theorem}
    Consider a nonlinear system 
    \begin{align}
            \bm{x}_{k+1} &= \bm{f}_k(\bm{x}_k, \bm{u}_k) + \bm{w}_k + \bm{d}_k, \\
            \bm{y}_{k} &= \bm{H} \bm{x}_{k}+ \bm{v}_{k}
            \label{eq:nonlin_theorem}
    \end{align}
    where $\bm{f}(\cdot, \cdot)$ is a GP-enhanced prediction model, 
    $\bm{H}$ is the linear measurement model, ${\bm{w}_{k} \sim \pazocal{N}(\bm{0},\bm{Q}_k)}$ and ${\bm{v}_{k} \sim \pazocal{N}(\bm{0},\bm{R}_k)}$ denote process and measurement noise and $\bm{d}_k$ represents the GP prediction error bound satisfying 
    \begin{equation}
        Pr\big(\bm{d}_k \in \pazocal{E}(\bm{0},\Tilde{\bm{X}}),\ \forall \bm{x}\in\mathbb{X}\big) \geq 1-\delta. 
        \label{eq:Dmuhat}
    \end{equation}
    Then, given a previous estimate $\hat{\bm{x}}_k$ and a confidence set $\pazocal{C}_k= \pazocal{E}(\hat{\bm{x}}_{k},\hat{\bm{X}}_{k}) \oplus \pazocal{E}(\bm{0},s\bm{P}_{k})$ satisfying $\bm{x}_{k}\in\pazocal{C}_{k}$, where $\hat{\bm{X}}_{k}$ parameterizes the ellipsoidal set of means of the previous time step,
    applying  \Cref{alg:torqueresidualAKF} yields a posterior estimate $\hat{\bm{x}}_{k+1}$ with a confidence set $\pazocal{C}_{k+1}$ containing the true state ${\bm{x}}_{k+1}$ with probability ${Pr\big({\bm{x}}_{k+1}\in\pazocal{C}_{k+1}\big)\geq 1-\delta}$.
\end{theorem}
\begin{proof}
    The first-order approximation of $\bm{f}(\cdot, \cdot)$ yields 
    \begin{equation}
        \Bar{\bm{f}}_k(\bm{x}_k,\bm{u}_k)= \bm{f}_k(\hat{\bm{x}}_k, \bm{u}_k) + \bm{A}_k(\bm{x}_k-\hat{\bm{x}}_k) + \bm{B}_k\bm{u}_k,
        \label{eq:ekf_linearization}
    \end{equation}
    with linearization $\bm{A}_k\in\mathbb{R}^{d\times d}$ and $\bm{B}_k\in\mathbb{R}^{d\times r}$, resulting in the linear state propagation for the EKF prediction
    \begin{equation}
            \Bar{\bm{x}}_{k+1} = \Bar{\bm{f}}_k(\bm{x}_k,\bm{u}_k) + \bm{w}_k + \bm{d}_k.
        \label{eq:linearmappingwithdx}
    \end{equation}
    Since the current confidence set $\bm{x}_k\in\pazocal{C}_k$ is given, we can bound the linearization error $\bm{\varepsilon}_{k+1}=\bm{x}_{k+1} - \Bar{\bm{x}}_{k+1}$ to the set
    \begin{equation}
        \Big\{\bm{f}_k(\bm{x}_k, \bm{u}_k) - \Bar{\bm{f}}_k(\bm{x}_k,\bm{u}_k) \ |\ \bm{x}_k\in\pazocal{C}_k\Big\},
        \label{eq:linerrsetdx}
    \end{equation}
    which is independent of $\bm{w}_k$ and $\bm{d}_k$ since they cancel out.
    Following the derivations in \cite{noack2010bounding}, the linearization error can bounded by an over-approximating ellipsoid denoted as such 
    \begin{equation}
        \bm{\varepsilon}_{k+1}\in\pazocal{E}(\bm{0},\bm{X}_k^f),
        \label{eq:linearerrellipsoid}
    \end{equation}
    where the elements in $\bm{X}_k^f$ are computed by taking the maximum error in each dimension seperately \cite{noack2010bounding}.
    Finally, we propagate the previous set of means $\pazocal{E}(\hat{\bm{x}}_{k},\hat{\bm{X}}_k)$ using the linearized model \eqref{eq:linearmappingwithdx}, which, due to ellipsoids permitting affine transformations \cite{noack2010bounding}, yields the set of prior predictions
    \begin{align}
        \pazocal{E}({\bm{x}}_{k+1}^-,{\bm{X}}_{k+1}^-) = &\underbrace{\pazocal{E}\Big(\bm{f}_k(\hat{\bm{x}}_k, \bm{u}_k) + \bm{B}_k\bm{u}_k,\, \bm{A}_k\hat{\bm{X}}_k\bm{A}_k^\intercal\Big) }_{\text{linear propagation of } \pazocal{E}(\hat{\bm{x}}_{k},\hat{\bm{X}}_k)}  \nonumber \\ 
        &\hspace{0.3cm}\oplus \underbrace{\pazocal{E}\big(\bm{0},\Tilde{\bm{X}}\big)}_{\text{ GP prediction error}} \oplus\underbrace{\pazocal{E}\big(\bm{0},\bm{X}_k^f\big),}_{\text{linearization error}}
    \label{eq:predwLinErrsetandGPerrset}
    \end{align}
    where we bound the errors due to the GP prediction and linearization using \eqref{eq:Dmuhat} and \eqref{eq:linearerrellipsoid}. After computing the prior covariance $\bm{P}_{k+1}^-$ as described in \Cref{alg:torqueresidualAKF}, inserting the observation model $\bm{H}$ yields the Kalman gain 
    \begin{equation}
        \bm {K} _{k+1}=\bm {P} _{k+1}^- \bm {H}^\intercal \big( \bm {H}\bm {P} _{k+1}^- \bm {H}^\intercal +\bm {R}_{k+1}\big)^{-1}
    \end{equation}
    Thus, the filtering step is carried out on the prior prediction set $\pazocal{E}({\bm{x}}_{k+1}^-,{\bm{X}}_{k+1}^-)$ to obtain a set of posterior means
    \begin{align}
    \pazocal{E}(\hat{\bm{x}}_{k+1},\hat{\bm{X}}_{k+1})=(\bm{I}-\bm{K}_{k+1}\bm{H})&\pazocal{E}(\hat{\bm{x}}^-_{k+1},\hat{\bm{X}}^-_{k+1}) \nonumber\\ &+ \bm {K} _{k+1}  \bm {y} _{k+1} \nonumber
    \end{align}
    Finally, with posterior covariance ${\bm{P}_{k+1}\!\!=\!(\bm{I}\!\!-\!\bm{K}_{k+1}\!\bm{H})\bm {P} _{k+1}^-}$, we obtain a confidence set 
    \begin{align}
        \pazocal{C}_{k+1} = \pazocal{E}(\hat{\bm{x}}_{k+1},\hat{\bm{X}}_{k+1}) \oplus \pazocal{E}(\bm{0},s\bm{P}_{k+1}),
    \end{align}
    which, given an appropriate choice of scalar $s$ determined by the probability level ${\delta}$, contains ${\bm{x}}_{k+1}$ with probability ${Pr\big({\bm{x}}_{k+1}\in\pazocal{C}_{k+1}\big)\geq 1-\delta}$, concluding the proof.
\end{proof}

Intuitively, \Cref{theorem} propagates the confidence set $\pazocal{C}_{k}$ for the Gaussian distributed state $\bm{x}_k$ by linearly propagating the set of means $\pazocal{E}(\hat{\bm{x}}_{k},\hat{\bm{X}}_k)$ using a first-order linearization and separately considering the maximum impact of neglected nonlinearities and GP prediction errors as additive and bounded perturbations. Since each term is represented by an ellipsoid, an outer-approximating ellipsoid can be obtained, which represents the propagated set of means. By additionally propagating the covariance matrix, stochastic uncertainties due to process and measurement noise are considered, thus, allowing us to obtain a probabilistic bound for the posterior estimate. Note that while we derive the estimation guarantees for a linear measurement model, it is straightforward to extend the procedure to nonlinear observation models by again applying a lineariztion and bounding the linearization error using over-approximating ellipsoid.
The overall scheme including the decomposition of error sources and uncertainty propagation is illustrated in \Cref{fig:errorboundscheme}.
\Cref{theorem} demonstrates that our proposed method permits error bounds for the estimated augmented-state, which in our considered application scenario translates to error bounds on the estimated active human torque. Thus, the GP-AKF is well-suited for use in conjunction with down-stream component, e.g., cooperative controllers, where reliable estimates are essential to guarantee safety of the interaction. \looseness=-1

\begin{figure}[t!]
	\centering
	\includegraphics[width=0.48\textwidth]{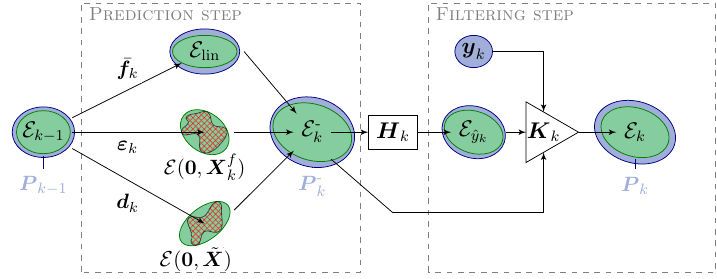}
	\caption[Ellipsoid error bounding scheme]{Ellipsoid error bounding scheme. The sets of means are represented in green color, Gaussian densities bounded to a certain sigma level are represented by blue color.}
	\label{fig:errorboundscheme}
\end{figure}

\section{Experimental Evaluation}

For the evaluation of the proposed learning-based force observer, we perform a human-robot interaction experiment with a one DoF SEA elbow-exoskeleton introduced in \Cref{subsec:experimentsetup}. 
First, we demonstrate that the model inference correctly learns the human passive dynamics in \Cref{subsec:pashumanwithoutweight}. 
Then our proposed method is successfully applied to estimate dynamic torques exerted by a subject in \Cref{subsec:acthumanstatictask}.

\subsection{Human-Exoskeleton Interaction Experimental Setup} \label{subsec:experimentsetup}

\textcolor{black}{The experiments are executed on an one DoF version of a SEA-driven shoulder-elbow exoskeleton introduced in\cite{Ercolini_Trigili_Baldoni_Crea_Vitiello_2019}.  The actuation unit is equipped with a DC brushless motor, a 1:100 reduction stage and a custom torsional spring running 
with a sampling rate of 100Hz for data aggregation. The load side angle $q$ and motor side angle $\theta_m$ are measured using two 19-bit absolute encoders and the motor torque $\tau_m$ is computed from the control input and torque constant.} Angular velocity and acceleration are obtained by numerical differentiation of the encoder signals together with a moving average filter. All experiments are performed on an Intel i5-10500 with 3.1GHz running Matlab2021b and LabVIEW2018. \looseness=-1

In our setup, the exoskeleton is attached to the subject's forearm with a cuff, while the torso is strapped to a fixed frame. Thereby, the participant's shoulder joint is fixed in place, gravitational torques of the upper arm are removed and only elbow flexion/extension movements can be performed. 
Each experiment consists of an online training phase on the task followed by an estimation phase. 
During the training phase the subject is instructed to remain passive, i.e., not resist or support the robot by intentional muscle activation. 
To minimize inconvenient calibration time, we perform online learning of the unknown dynamics during the training phase.
Additionally, due to the intrinsic compliance properties and torque limits of the SEA, it is guaranteed that the interaction forces applied to the human arm remain in safe regions. 

\subsection{Accuracy with Human Passive Dynamics Online Learning} \label{subsec:pashumanwithoutweight}

First, we demonstrate online model inference of the human passive dynamics and the principle applicability of \Cref{ass:passive_human}. 
To this end, we perform training and estimation on the same trajectory and instruct the participant to remain passive during both stages. Thus, given an accurate model inference during training, the estimated human torque is exactly $0$. 

In particular, we use sequences of sigmoid functions as reference trajectory to generate data due to their minimum-jerk profile, which are tracked by a PID position controller. Each sequence starts at an initial position $10\deg$ and moves to the target position $75\deg$ in $3.5$s. Repeating this sequence forwards and backwards 6 times yields a total of $N=2100$ training samples. For online learning we use the locally growing random tree of GPs (LoG-GP) approach \cite{Lederer2021b} with a maximum number $\bar{N}=100$ training samples per local model and the gradient-based method from \cite{TESFAZGI2023501} to optimize the hyperparameters of the GPs online during training. For our GP-AKF method, we determine the measurement noise matrix using the autocovariance least-squares method \cite{ODELSON2006303} to 
\begin{equation}
\boldsymbol{R}= \text{diag}\big(0.0461,\ 3.6\cdot10^{-6},\ 0.1288,\ 1.01\cdot10^{-5}\big),
\end{equation}
while we set the nominal process noise to 
\begin{equation}
\boldsymbol{\Sigma}_\text{nom}= \text{diag}\big(10^{-1},\ 10^{3},\ 10^{2},\ 10^{-9},\ 10^{-2}\big).
\end{equation}
The initial state $\hat{\boldsymbol{x}}_0$ is measured and accordingly the initial covariance $\boldsymbol{P}_0$ is set to the measurement noise covariance $\boldsymbol{R}$.

We compare our proposed GP-enhanced torque observer to a standard augmented-state KF (AKF) \cite{park2017seakf} and to the spring torque $\tau_s(\cdot,\cdot)$ \eqref{eq:spring_torque}, which is frequently used as a measure of the interaction torques \cite{sea95}. The estimated active human torques $\tau_{\text{h,act}}$ for all methods are depicted in \Cref{fig:zeroestimate}. Furthermore, a quantitative comparison of the estimation accuracy is provided in \Cref{tab:pass_torque}, where the achieved root mean square error (RMSE) is shown. It is clearly visible that both the AKF method and spring torque $\tau_s(\cdot,\cdot)$ fail to provide an accurate estimate of the active human torque. For the AKF this is primarily due to static friction effects that are not accounted for by the nominal dynamics model, thus, inducing large jumps in the active torque at time steps coinciding with turning points of the trajectory. While the spring torque is not effected by this friction, it still produces inaccurate results due to the non-negligible load-side dynamics. Finally, our proposed approach is able to recover an accurate estimate of the actively applied human torque during estimation, which implies that a correct model of the human passive dynamics was learned during the online training phase. Moreover, it can be seen that the uncertainty of the GP-AKFs shrinks quickly from the conservative initial value to a tight region around the estimate and indicates high confidence in the learned dynamics. The accurate estimation implies that the online model inference from \cref{subsec:online_learning} not only successfully learns the passive human model but even additional unmodelled effects such as friction.

\begin{figure}[t!]
    \centering
        \includegraphics[width=0.475\textwidth]{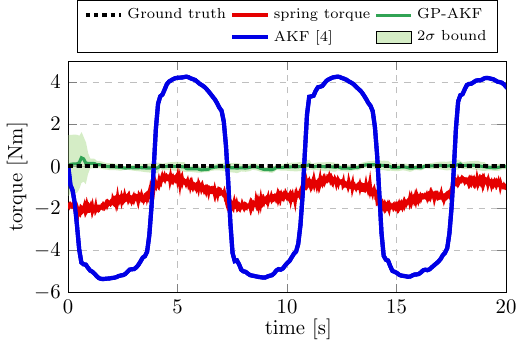}
         \caption[Estimated human torque for passive subject.]
         {Estimated active human torque $\hat{\tau}_\text{h,act}$ for a passively acting subject. The ground truth (zero) is depicted with a black, dotted line. Our proposed GP-AKF (green) estimates the zero torque correctly based on the learned passive human dynamics, while the standard AKF (dark blue) overshoots due to unmodelled friction effects. The spring torque (red) also provides inaccurate estimates. $2\sigma$ bounds based on the estimate covariance $\boldsymbol{P}_k$ are represented by the shaded areas.}
         \label{fig:zeroestimate}
\end{figure}

\textcolor{blue}{
\begin{table}[!t]
	\center
	\caption{RMSE for the estimated active human torque $\hat{\tau}_\text{h,act}$ for a passively acting subject. }
		\begin{sc}
			\begin{tabular}{l c c c}
				\toprule
				 & GP-AKF& AKF & spring torque \\
				\midrule
				RMSE [Nm]& $0.067$ & $4.175$ & $1.359$\\
				\bottomrule
			\end{tabular}
		\end{sc}
	\label{tab:pass_torque}
\end{table}
}

\vspace{-0.55cm}
\subsection{Estimation of Dynamic Human Torques} \label{subsec:acthumanstatictask}
To demonstrate the accurate estimation of dynamic torques exerted by the subject, we design an experiment where active and passive human torques can be decomposed such that the ground truth is available. 
In particular, the task is performed in a static position, e.g., $q_0=10\deg$, while the exoskeleton is in open loop torque-control. The motor-torque is set to $\tau_m = \tau_\text{des} + \tau_\text{comp}$, with $\tau_\text{comp}$ configured to compensate the static gravity and friction torques of the human-exoskeleton system and $\tau_\text{des}$ represents the desired active torque.

During training stage  $\tau_\text{des}=0$, thus, no joint movement is induced, since $\tau_\text{comp}$ compensates for all load-side dynamics. However, during the estimation phase, a torque profile ${\tau_\text{des} = \left(-2\cdot \cos\left(2\pi f\cdot t\right) + 2 \right),}$
with frequency $f=0.1$Hz is set. Since $\tau_\text{comp}$ already compensates the the load-side gravity and friction, $\tau_\text{des}$ 
perturbs the arm from the initial position ${q_0=10\deg}$. The subject is instructed to resist the perturbation and hold the arm in the static initial position. Accordingly, 
the active human torque then becomes ${\tau_{\text{h,act}}=-\tau_\text{des}}$. The training stage is carried out online for a duration of six seconds, yielding $N=600$ training samples.

\begin{figure}[t!]
    \centering
         \includegraphics[width=0.475\textwidth]{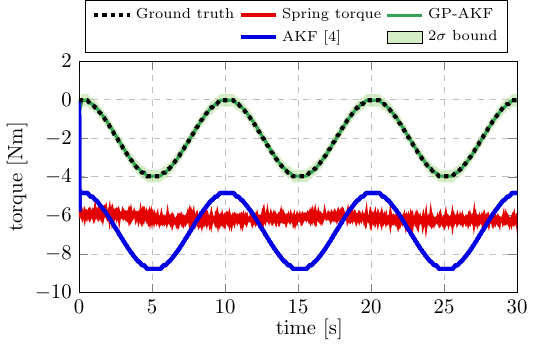}
         \caption[Torque estimate for active subject resisting the exoskeleton.]{Torque estimate $\hat{\tau}_\text{h,act}$ for an active subject. Ground truth is depicted in black, dotted. The GP-AKF (green) estimates the dynamic torque accurately based on the learned human dynamics. The AKF (dark blue) has a significant offset due to the unknown human dynamics and static friction. The spring torque (red) does not represent the active torque due to saturation effects of the spring deformation. }
         \label{fig:staticexperiment}
\end{figure}%

\begin{table}[!t]
	\center
	\caption{RMSE for the estimated active human torque $\hat{\tau}_\text{h,act}$ for an active subject resisting the perturbation torque. }
		\begin{sc}
			\begin{tabular}{l c c c}
				\toprule
				 & GP-AKF& AKF & spring torque \\
				\midrule
				RMSE [Nm]& $0.017$ & $4.805$ & $4.389$\\
				\bottomrule
			\end{tabular}
		\end{sc}
	\label{tab:act_torque}
\end{table}

The results of the described experiment are shown in \Cref{fig:staticexperiment} and \Cref{tab:act_torque}. We observe that our GP-AKF method achieves high accuracy in estimating the torque with which the human resists the perturbation. 
On the other hand, the AKF method (blue) does not recover the true active torque $\tau_\text{h,act}$. Due to the unmodelled passive dynamics, the AFK wrongfully estimates that the subject actively generates torques to compensate for the full motor torque $\tau_m$ instead of just $\tau_\text{des}$, 
resulting in the offset visible in \Cref{fig:staticexperiment}. 
In comparison to the observer approaches, the spring torque does not follow the periodic profile of the true human torque in \Cref{fig:staticexperiment}. This is due to nonlinear saturation effects in the spring deformation, thus, leading to poor estimates of interaction forces using the available spring torque model. Hence, only our proposed GP-AKF approach is able to provide accurate estimates of the actively exerted human torques by inferring a correct model of the human-exoskeleton dynamics online. \looseness=-1

\section{CONCLUSIONS}
In this paper, we propose a novel approach for estimating the actively applied human torque during interaction with an elastic joint robot by combining GP regression with and augmented-state KF. 
Through the uncertainty quantification of the GP, we obtain a confidence measure of the state estimate that considers local model uncertainty and additionally derive guaranteed estimation error bounds for the GP-BayesFilter. Finally, the effectiveness of the approach is demonstrated in human-exoskeleton experiments.

\bibliographystyle{IEEEtran}
\bibliography{myBib}

\begin{thebibliography}{10}
\providecommand{\url}[1]{#1}
\csname url@samestyle\endcsname
\providecommand{\newblock}{\relax}
\providecommand{\bibinfo}[2]{#2}
\providecommand{\BIBentrySTDinterwordspacing}{\spaceskip=0pt\relax}
\providecommand{\BIBentryALTinterwordstretchfactor}{4}
\providecommand{\BIBentryALTinterwordspacing}{\spaceskip=\fontdimen2\font plus
\BIBentryALTinterwordstretchfactor\fontdimen3\font minus
  \fontdimen4\font\relax}
\providecommand{\BIBforeignlanguage}[2]{{%
\expandafter\ifx\csname l@#1\endcsname\relax
\typeout{** WARNING: IEEEtran.bst: No hyphenation pattern has been}%
\typeout{** loaded for the language `#1'. Using the pattern for}%
\typeout{** the default language instead.}%
\else
\language=\csname l@#1\endcsname
\fi
#2}}
\providecommand{\BIBdecl}{\relax}
\BIBdecl

\bibitem{yu2015seadob}
H.~Yu, S.~Huang, G.~Chen, Y.~Pan, and Z.~Guo, ``Human-robot interaction control
  of rehabilitation robots with series elastic actuators,'' \emph{IEEE
  Transactions on Robotics}, vol.~31, no.~5, pp. 1089--1100, 2015.

\bibitem{sea95}
G.~Pratt and M.~Williamson, ``Series elastic actuators,'' in \emph{IEEE/RSJ
  International Conference on Intelligent Robots and Systems. Human Robot
  Interaction and Cooperative Robots}, 1995, pp. 399--406 vol.1.

\bibitem{emilio}
A.~Pilla, E.~Trigili, Z.~McKinney, C.~Fanciullacci, C.~Malasoma, F.~Posteraro,
  S.~Crea, and N.~Vitiello, ``Robotic rehabilitation and multimodal
  instrumented assessment of post-stroke elbow motor functions—a randomized
  controlled trial protocol,'' \emph{Frontiers in Neurology}, vol.~11, 2020.

\bibitem{park2017seakf}
Y.~Park, N.~Paine, and S.~Oh, ``Development of force observer in series elastic
  actuator for dynamic control,'' \emph{IEEE Transactions on Industrial
  Electronics}, vol.~65, no.~3, pp. 2398--2407, 2017.

\bibitem{manipulation}
X.~Li, Y.~Pan, G.~Chen, and H.~Yu, ``Adaptive human–robot interaction control
  for robots driven by series elastic actuators,'' \emph{IEEE Transactions on
  Robotics}, vol.~33, no.~1, pp. 169--182, 2017.

\bibitem{kalman1961new}
R.~E. Kalman and R.~S. Bucy, ``New results in linear filtering and prediction
  theory,'' 1961.

\bibitem{liu2019momentumkf}
L.-K. Liu, T.-C. Chien, Y.-L. Chen, L.-C. Fu, and J.-S. Lai, ``Sensorless
  control with friction and human intention estimation of exoskeleton robot for
  upper-limb rehabilitation,'' in \emph{IEEE International Conference on
  Robotics and Biomimetics}.\hskip 1em plus 0.5em minus 0.4em\relax IEEE, 2019,
  pp. 290--296.

\bibitem{liu2021sensorless}
S.~Liu, L.~Wang, and X.~V. Wang, ``Sensorless force estimation for industrial
  robots using disturbance observer and neural learning of friction
  approximation,'' \emph{Robotics and Computer-Integrated Manufacturing},
  vol.~71, p. 102168, 2021.

\bibitem{Rasmussen2006}
C.~E. Rasmussen and C.~K.~I. Williams, \emph{{Gaussian Processes for Machine
  Learning}}.\hskip 1em plus 0.5em minus 0.4em\relax Cambridge, MA: The MIT
  Press, 2006.

\bibitem{GPbounds}
A.~Lederer, J.~Umlauft, and S.~Hirche, ``Uniform error bounds for gaussian
  process regression with application to safe control,'' in \emph{Advances in
  Neural Information Processing Systems}, vol.~32, 2019, pp. 659--669.

\bibitem{santina2020softrobotest}
C.~Della~Santina, R.~L. Truby, and D.~Rus, ``Data--driven disturbance observers
  for estimating external forces on soft robots,'' \emph{IEEE Robotics and
  automation letters}, vol.~5, no.~4, pp. 5717--5724, 2020.

\bibitem{Evangelisti}
G.~Evangelisti and S.~Hirche, ``Data-driven momentum observers with physically
  consistent gaussian processes,'' \emph{IEEE Transactions on Robotics}, pp.
  1--14, 2024.

\bibitem{ko2009gp}
J.~Ko and D.~Fox, ``Gp-bayesfilters: Bayesian filtering using gaussian process
  prediction and observation models,'' \emph{Autonomous Robots}, vol.~27,
  no.~1, pp. 75--90, 2009.

\bibitem{noack2010bounding}
B.~Noack, V.~Klumpp, N.~Petkov, and U.~D. Hanebeck, ``Bounding linearization
  errors with sets of densities in approximate kalman filtering,'' in
  \emph{IEEE International Conference on Information Fusion}, 2010, pp. 1--8.

\bibitem{Spong1987}
M.~W. Spong, ``{Modeling and Control of Elastic Joint Robots},'' \emph{Journal
  of Dynamic Systems, Measurement, and Control}, vol. 109, no.~4, pp. 310--318,
  1987.

\bibitem{just2020human}
F.~Just, {\"O}.~{\"O}zen, S.~Tortora, V.~Klamroth-Marganska, R.~Riener, and
  G.~Rauter, ``Human arm weight compensation in rehabilitation robotics:
  efficacy of three distinct methods,'' \emph{Journal of neuroengineering and
  rehabilitation}, vol.~17, no.~1, pp. 1--17, 2020.

\bibitem{radke2006survey}
A.~Radke and Z.~Gao, ``A survey of state and disturbance observers for
  practitioners,'' in \emph{American Control Conference}, 2006, pp. 6--pp.

\bibitem{illposed}
T.~Oomen and O.~Bosgra, ``Estimating disturbances and model uncertainty in
  model validation for robust control,'' in \emph{IEEE Conference on Decision
  and Control}, 2008, pp. 5513--5518.

\bibitem{wahrburg2015momentumkf}
A.~Wahrburg, E.~Morara, G.~Cesari, B.~Matthias, and H.~Ding, ``Cartesian
  contact force estimation for robotic manipulators using kalman filters and
  the generalized momentum,'' in \emph{IEEE International Conference on
  Automation Science and Engineering}, 2015, pp. 1230--1235.

\bibitem{Srinivas2012}
N.~Srinivas, A.~Krause, S.~M. Kakade, and M.~W. Seeger,
  ``{Information-Theoretic Regret Bounds for Gaussian Process Optimization in
  the Bandit Setting},'' \emph{IEEE Transactions on Information Theory},
  vol.~58, no.~5, pp. 3250--3265, 2012.

\bibitem{Ercolini_Trigili_Baldoni_Crea_Vitiello_2019}
G.~Ercolini, E.~Trigili, A.~Baldoni, S.~Crea, and N.~Vitiello, ``A novel
  generation of ergonomic upper-limb wearable robots: Design challenges and
  solutions,'' \emph{Robotica}, vol.~37, p. 2056–2072, 2019.

\bibitem{Lederer2021b}
A.~Lederer, A.~{Odonez Conejo}, K.~Maier, W.~Xiao, J.~Umlauft, and S.~Hirche,
  ``{Gaussian Process-Based Real-Time Learning for Safety Critical
  Applications},'' in \emph{International Conference on Machine Learning},
  2021, pp. 6055--6064.

\bibitem{TESFAZGI2023501}
S.~Tesfazgi, A.~Lederer, J.~F. Kunz, A.~J. Ordóñez-Conejo, and S.~Hirche,
  ``Model-based robot control with gaussian process online learning: An
  experimental demonstration,'' \emph{IFAC-PapersOnLine}, vol.~56, no.~2, pp.
  501--506, 2023.

\bibitem{ODELSON2006303}
B.~J. Odelson, M.~R. Rajamani, and J.~B. Rawlings, ``A new autocovariance
  least-squares method for estimating noise covariances,'' \emph{Automatica},
  vol.~42, no.~2, pp. 303--308, 2006.

\end{thebibliography}

\end{document}